%% file: ddpriors.tex
\documentclass{article}
\usepackage[final,nonatbib]{nips_2018}
\usepackage[usenames,dvipsnames]{xcolor} %

\usepackage{mathptmx} 

\usepackage{lineno}
\usepackage{floatpag}
\usepackage{amsthm}

\input{headers.tex}
\declaretheorem[style=plain,numberwithin=section,name=Theorem]{theorem}
\declaretheorem[style=plain,sibling=theorem,name=Lemma]{lemma}

\declaretheorem[style=plain,sibling=theorem,name=Corollary]{corollary}

\declaretheorem[style=definition,sibling=theorem,name=Definition]{definition}

\numberwithin{equation}{section}
\numberwithin{theorem}{section}

\usepackage{tikz}
\usetikzlibrary{decorations.pathreplacing} 

\usepackage{enumitem}
\usepackage[T1]{fontenc}    %
\usepackage[utf8]{inputenc} %
\usepackage{url}            %
\usepackage{booktabs}       %
\usepackage{amsfonts}       %

\usepackage[colorlinks,citecolor=blue,urlcolor=Blue,linkcolor=blue]{hyperref}

\usepackage[capitalize]{cleveref}

\crefname{lemma}{Lemma}{Lemmas}
\crefname{corollary}{Corollary}{Corollaries}
\crefname{theorem}{Theorem}{Theorems}

\input{defs.tex}

\newcommand{\Qt}[1]{Q_{#1}}
\newcommand{\Zt}[1]{Z_{#1}}
\newcommand{\Lpnorm}[3]{\norm{#3}_{L^{#1}(#2)}}
\newcommand{\mineig}{\sigma_{\mathrm{min}}}
\newcommand{\SigmaInv}{\Sigma^{-1}}

\newcommand{\invtemp}{inverse temperature}

\title{Data-dependent PAC-Bayes priors \\via differential privacy}
\author{%
Gintare Karolina Dziugaite
\\
University of Cambridge; Element AI
\And
Daniel M. Roy\\
University of Toronto; Vector Institute
}

\begin{document}

\maketitle

\begin{abstract}
The Probably Approximately Correct (PAC) Bayes framework \citep{PACBayes} can incorporate knowledge about 
the learning algorithm and (data) distribution through the use of distribution-dependent priors, 
yielding tighter generalization bounds on data-dependent posteriors.
Using this flexibility, however, is difficult, especially when the data distribution is presumed to be unknown.
We show how an $\epsilon$-differentially private data-dependent prior yields a valid PAC-Bayes bound,
and then show how non-private mechanisms for choosing priors can also yield generalization bounds. 
As an application of this result, we show that a Gaussian prior mean chosen via 
stochastic gradient Langevin dynamics \citep[SGLD;][]{welling2011bayesian} 
 leads to a valid PAC-Bayes bound given control of the 2-Wasserstein distance to an 
 $\epsilon$-differentially private stationary distribution.
We study our data-dependent bounds empirically, and show 
that they can be nonvacuous even when other distribution-dependent bounds are vacuous.
\end{abstract}

\section{Introduction}

There has been a resurgence of interest in PAC-Bayes bounds,
especially towards explaining generalization in large-scale neural networks trained by stochastic gradient descent \citep{DR17,neyshabur2017exploring,neyshabur2017pac,NIPS2017_6886}. 
See also \citep{begin2016pac,germain2016pac,thiemann2017strongly,bartlett2017spectrally,RRT17,grunwald2017tight,smith2018bayesian}.

PAC-Bayes bounds control the generalization error of Gibbs classifiers (aka PAC-Bayes ``posteriors'') 
in terms of the Kullback--Leibler (KL) divergence to a fixed probability measure (aka PAC-Bayes ``prior'') on the space of classifiers.
PAC-Bayes bounds depend on a tradeoff between the empirical risk of the posterior $Q$ and a penalty $\frac 1 m \KL{Q}{P}$, 
where $P$ is the prior, fixed independently of the sample $S \in Z^m$ from some space $Z$ of labelled examples. 
The KL penalty is typically the largest contribution to the bound and so finding the tightest possible bound generally depends on minimizing the KL term.

The KL penalty vanishes for $Q=P$, but typically $P$, viewed as a randomized (Gibbs) classifier, has poor performance since it has been chosen independently of the data. On the other hand, since $P$ is chosen independently of the data, posteriors $Q$ tuned to the data to achieve minimal empirical risk often bear little resemblance to the data-independent prior $P$, causing $\KL{Q}{P}$ to be large. 
As a result, PAC-Bayes bounds can be loose or even vacuous.

The problem of excessive KL penalties is \emph{not} inherent to the PAC-Bayes framework.
Indeed, the PAC-Bayes theorem permits one to choose the prior $P$ based on the distribution $\Dist$ of the data. 
However, since $\Dist$ is considered unknown, and our only insight as to $\Dist$ is through the sample $S$, this flexibility would seem to be useless, as $P$ must be chosen independently of $S$ in existing bounds.
Nevertheless, it is possible to make progress in this direction, and it is likely the best way towards tighter bounds and deeper understanding.

There is a growing body of work in the PAC-Bayes literature on data-distribution-dependent priors \citep{Catoni,parrado2012pac,LEVER2013}.
Our focus is on generalization bounds that use all the data, and in this setting, 
\citet{LEVER2013} prove a remarkable result for Gibbs posteriors.
Writing $\EmpRisk{S}{}$ for the empirical risk function with respect to the sample $S$,
\citeauthor{LEVER2013} study the randomized classifier $Q$ with density (relative to some base measure) proportional to $\exp \parens{- \tau \EmpRisk{S}{}}$. 
For large values of $\tau$, $Q$ concentrates increasingly around the empirical risk minimizers.
As a prior, the authors consider the probability distribution with density proportional to $\exp \parens{- \tau \Risk{\Dist}{}}$,
where the empirical risk has been replaced by its expectation, $\Risk{\Dist}{}$, the (true) risk.
Remarkably, \citeauthor{LEVER2013} are able to upper bound the KL divergence between the two distributions by a term
independent of $\Dist$, yielding the following result: Here $\KLbin{q}{p}$ is the KL divergence between Bernoulli measures with mean $q$ and $p$.
See \cref{dppbproof} for more details.

\begin{theorem}[{\citealt{LEVER2013}}]\label{leverbound}
Fix $\tau > 0$.
For $S \in Z^m$, let $Q(S) = \Gibbs{P}{- \tau \EmpRisk{S}{}}$ be a Gibbs posterior with respect to some base measure $P$ on $\HH$, 
where the empirical risk $\EmpRisk{S}{}$ is bounded in $[0,1]$.
For every $\delta > 0$, $m \in \Nats$, distribution $\Dist$ on $Z$,
\begin{equation}
\PPr{S \sim \Dist^m} 
    \Bigl ( 
      \KLbin{\EmpRisk{S}{Q(S)}}{\Risk{\Dist}{Q(S)}} \le \frac 1 m \parens[\Big]
         { \tau \sqrt{ \frac 2 m \ln \frac {2 \sqrt{m}}{\delta} + \frac{\tau^2}{2m} + \ln \frac {2\sqrt{m}}{\delta}} }
    \Bigr ) \ge 1-\delta.
\end{equation}
\end{theorem}

The dependence on the data distribution is captured through $\tau$, which is ideally chosen as small as possible, subject to $Q(S)$ yielding small empirical risk. (One can use a union bound to tune $\tau$ based on $S$.)
The fact that the KL bound does not depend on $\Dist$, other than through $\tau$,
implies
that the bound \emph{must} be loose for all $\tau$ such that there exists a distribution $\Dist$ that causes $Q$ to overfit with high probability on size $m$ datasets $S \sim \Dist^m$.
In other words, for fixed $\tau$, the bound is no longer distribution dependent.
This would not be important if not for the following empirical finding: weights sampled according to high values of $\tau$ do not overfit on real data, but they do on data whose labels have been randomized. Thus these bounds are vacuous in practice when the generalization error is, in fact, small.
Evidently, the KL bound gives up too much.

Our work launches a different attack on the problem of using distribution-dependent priors.
Loosely speaking, if a prior is chosen on the basis of the data, but in a way that is very stable to perturbations of the data set, 
then the resulting data-dependent prior should reflect the underlying data distribution, rather than the data, resulting in a bound that should still hold, perhaps with smaller probability.

We formalize this intuition using differential privacy \citep{Dwork2006,dwork2015preserving}. 
We show that an $\epsilon$-differentially private prior mean yields a valid, 
though necessarily looser, PAC-Bayes generalization bound. (See \cref{DPpacbayes}.)
The result is a straightforward application of results connecting privacy and adaptive data analysis \citep{dwork2015preserving,dwork2015generalization}.

The real challenge is using such a result: 
In practice, $\epsilon$-differentially private mechanisms can be expensive to compute.
In the context of generalization bounds for neural networks, 
we consider the possibility of using stochastic gradient Langevin dynamics \citep[SGLD;][]{welling2011bayesian}
to choose a data-dependent prior by way of stochastic optimization/sampling. 

By various results, SGLD is known to produce an $(\epsilon,\delta)$-differentially private release~
\citep{mir2013differential,bassily2014differentially,dimitrakakis2014robust,Wang:2015,Minami16}
A gap remains between pure and approximate differential privacy.
Even if this gap were to be closed, 
the privacy/accuracy tradeoff of these analyses is too poor because they do not take  
advantage of the fact that, under some technical conditions, the distribution of SGLD's output 
converges weakly towards a stationary distribution \citep[Thm.~7]{teh2016consistency}, which is $\epsilon$-differentially private.
One can also bound the KL divergence (and then 2-Wasserstein distance) of SGLD to stationarity within a constant given an appropriate fixed step size \citep{RRT17}.
Neither result implies that SGLD achieves pure $\epsilon$-differential privacy.

We show that we can circumvent this barrier in our PAC-Bayes setting. 
We give a general PAC-Bayes bound for non-private data-dependent priors and then an application to multivariate Gaussian priors with non-private data-dependent means, with explicit bounds for the case of Gibbs posteriors. 
In particular, conditional on a data set, if a data-dependent mean vector $\ww$ is close in 2-Wasserstein distance
to an $\epsilon$-differentially private mean vector, 
then the generalization errors is close to that of the $\epsilon$-differentially private mean. 
The data-dependent mean $\ww$
is not necessarily differentially private, even approximately.
As a consequence, under suitable assumptions, SGLD can be used to optimize a data-dependent mean and still yield a valid PAC-Bayes bound.

\section{Other Related Work}

Our analysis relies on the stability of a data-dependent prior.  
Stability has long been understood to relate to generalization \citep{bousquet2002stability}.
Our result relies on the connection between generalization and differential privacy \citep{dwork2015preserving,dwork2015generalization,bassily2016algorithmic,Oneto2017},
which can be viewed as a particularly stringent notion of stability.
See \citep{dwork2008differential} for a survey of differential privacy.

\citet[Thm.~1]{kifer2012private} also establish a ``limit'' theorem for differential privacy, showing that the almost sure convergence
of mechanisms of the same privacy level admits a private mechanism of the same privacy level. Our result can be 
viewed as a significant weakening of the hypothesis to require only that the weak limit be private: no element on the sequence need 
be private.

The bounds we establish hold for bounded loss functions and i.i.d.\ data. 
Under additional assumptions, one can obtain PAC-Bayes generalization and excess risk bounds 
for unbounded loss with heavy tails \citep{Catoni, germain2016pac, grunwald2016fast, alquier2018}. 
\citet{alquier2018} also consider non-i.i.d.\ training data. Our approach to differentially private data-dependent priors 
can be readily extended to these settings.

\section{Preliminaries}

Let $Z$ be a measurable space, 
let $\ProbMeasures{Z}$ denote the space of probability measures on $Z$,
and let $\Dist \in \ProbMeasures{Z}$ be unknown.
We consider the batch supervised learning setting under a loss function bounded below:
having observed $S \sim \Dist^m$, i.e., $m$ independent and identically distributed samples from
$\Dist$, 
we aim to choose a predictor, parameterized by weight vector $\ww \in \HH$,
with minimal \emph{risk} 
\begin{equation*}
\Risk{\Dist}{\ww} \defas\EEE{z \sim \Dist} \event {\loss(\ww,z) },
\end{equation*}
where $\loss : \HH \times Z \to \Reals$ is measurable and bounded below.
(We ignore the possibility of constraints on the weight vector for simplicity.)
We also consider randomized predictors,
represented by probability measures $Q \in \RHH$,
whose risks are defined via averaging,
\begin{equation*}
\Risk{\Dist}{Q} 
\defas 
\smash
{\EEE{\ww \sim Q} \event {\Risk{\Dist}{\ww} }
= \EEE{z \sim \Dist} \event[\Big]{ \EEE{\ww \sim Q} \event {\loss(\ww,z)} }
},
\end{equation*}
where the second equality follows from Tonelli's theorem and the fact that $\loss$ is bounded below.

Let $S = (z_1,\dots,z_{m})$ and let $\EmpDist \defas \frac 1 m \sum_{i=1}^m \delta_{z_i}$ be the empirical distribution.
Given a randomized predictor $Q$, 
such as that chosen by a learning algorithm on the basis of data $S$,
its \emph{empirical risk}
\begin{equation*}
\textstyle \EmpRisk{S}{Q} 
\defas \Risk{\EmpDist}{Q} = 
\smash {\frac 1 m \sum_{i=1}^m \EEE{\ww \sim Q} \event {\loss(\ww,z_i)}},
\end{equation*}
is studied as a stand-in for its risk, which we cannot compute.
While $\EmpRisk{S}{Q}$ is easily seen to be an unbiased estimate of $\Risk{\Dist}{Q}$ 
when $Q$ is independent of $S$, 
our goal is to characterize the (one-sided) \emph{generalization error}
$
\Risk{\Dist}{Q} - \EmpRisk{S}{Q}
$
when $Q$ is random and dependent on $S$.

Finally, when optimizing the weight vector or defining tractable distributions on $\HH$,
we use a (differentiable) surrogate risk $\SurRisk{S}{}$,
which is the empirical average of a bounded surrogate loss, taking values in an interval of length $\Delta$.

\subsection{Differential privacy} 
\label{DPbackground}

For readers not familiar with differential privacy, \cref{DPintro} provides the basic definitions and results.
We use the notation $\Alg\colon R \randto T$ to denote a randomized algorithm $\Alg$ that takes an input in a measurable space $R$ and produces a random output in the measurable space $T$. 

\begin{definition}
A randomized algorithm $\Alg\colon Z^m \randto T$ is 
\defn{$(\epsilon,\delta)$-differentially private}
if, for all pairs $S,S' \in Z^m$ that differ at only one coordinate,
and all measurable subsets $B \subseteq T$,
we have
$
\Pr \event{ \Alg(S) \in B }
\le 
\e^\epsilon \,
\Pr \event { \Alg(S') \in B } + \delta.
$
Further, $\epsilon$-differentially private means $(\epsilon,0)$-differentially private.
\end{definition}

For our purposes, \emph{max-information} is the key quantity controlled by differential privacy.
\begin{definition}[{\citealt[\S3]{dwork2015generalization}}]
Let $\beta \ge 0$, let $X$ and $Y$ be random variables in arbitrary measurable spaces, 
and let $X'$ be independent of $Y$ and equal in distribution to $X$.
The \defn{$\beta$-approximate max-information between $X$ and $Y$}, 
denoted $\amaxinf{\beta}{X}{Y}$,
is the least value $k$ such that, 
for all product-measurable events $E$,
\[
\Pr \set{(X,Y) \in E } \le \e^{k} \,\Pr \set{(X',Y) \in E} + \beta.
\]
The \defn{max-information} $\maxinf{X}{Y}$ is defined to be $\amaxinf{\beta}{X}{Y}$ for $\beta=0$. 
For $m \in \Nats$ and $\Alg\colon Z^m \randto T$, the \defn{$\beta$-approx.\ max-information of $\Alg$}, denoted $\amaxinfalg{\beta}{\Alg}{m}$, 
is the least value $k$ such that, for all $\Dist \in \ProbMeasures{Z}$,
$\amaxinf{\beta}{S}{\Alg(S)} \le k$ when $S \sim \Dist^m$. The max-information of $\Alg$ is defined similarly.\footnote{Note that in much of the literature it is standard to express the max-information in \emph{bits}, i.e., the factor $\e^k$ above is replaced by $2^{k'}$ with $k' = k \log_2 e$.
}
\end{definition}

In \cref{ddpsec}, we consider the case where the dataset $S$ and a data-dependent prior $\PAlg(S)$ have small approximate max-information. 
The above definition tells us that we can almost treat the data-dependent prior as if it was chosen independently from $S$.
The following is the key result connecting pure differential privacy and max-information:

\begin{theorem}[{\citealt[Thms.~19--20]{dwork2015generalization}}]
\label{dpthm}
Fix $m \in \Nats$.
Let $\Alg\colon Z^m \randto T$ be $\epsilon$-differentially private.
Then $\maxinfalg{\Alg}{m} \le \epsilon m$
and, for all $\beta > 0$, $\amaxinfalg{\beta}{\Alg}{m} \le \epsilon^2 m/2 + \epsilon \sqrt{m \ln (2/\beta)/2}$.
\end{theorem}

\section{PAC-Bayes bounds}
\label{dppbproof}

Let $Q,P \in \ProbMeasures{\HH}$.
When $Q$ is absolutely continuous with respect to $P$, written $Q \ll P$,
we write $\rnd{Q}{P} : \HH \to \NNReals \cup \{\infty\}$ for some Radon--Nikodym derivative of $Q$ with respect to $P$.
The Kullback--Liebler (KL) divergence 
from $Q$ to $P$ is 
$\KL{Q}{P} = \int \ln \rnd{Q}{P} \,\dee Q$ if $Q \ll P$ and $\infty$ otherwise.
Let $\Bernoulli{p}$ denote the Bernoulli distribution on $\{0,1\}$ with mean $p$.
For $p,q \in [0,1]$, we abuse notation and define
\begin{equation*}
\KLbin{q}{p} \defas 
\smash{ \KL{\Bernoulli{q}}{\Bernoulli{p}} = q \ln \frac q p + (1-q) \ln \frac {1-q}{1-p}. }
\end{equation*}

The following PAC-Bayes bound for bounded loss is due to \citet{Maurer04},
and extends the 0--1 bound established by \citet{LS01}, building 
off the seminal work of \citet{PACBayes} and \citet{STW97}.
See also \citep{LangfordPHD} and \citep{Catoni}.

\begin{theorem}[PAC-Bayes; {\citealt[Thm.~5]{Maurer04}}]
\label{pacbayes} 
Under bounded loss $\loss \in [0,1]$,
for every $\delta > 0$, $m \in \Nats$, distribution $\Dist$ on $Z$, and distribution $P$ on $\HH$,
\begin{equation}
\PPr{S \sim \Dist^m} 
    \Bigl ( 
      (\forall Q)\ 
      \KLbin{\EmpRisk{S}{Q}}{\Risk{\Dist}{Q}} \le \frac {\KL{Q}{P} + \ln \frac{2 \sqrt{m}}{\delta}}{m}  
    \Bigr ) \ge 1-\delta.
\end{equation}
\end{theorem}

One can use Pinsker's inequality to obtain a bound on the generalization error $\abs{ \EmpRisk{S}{Q} - \Risk{\Dist}{Q}}$, however this 
significantly loosens the bound, especially when $\EmpRisk{S}{Q}$ is close to zero. 
We refer to the quantity $\KLbin{\EmpRisk{S}{Q}}{\Risk{\Dist}{Q}}$ as the \defn{KL-generalization error}.
From a bound on this quantity, we can bound the risk as follows:
given empirical risk $q$ and a bound on the KL-generalization error $c$, 
the risk is bounded by the largest value $p \in [0,1]$ such  $ \KLbin{q}{p} \le c$. 
See \citep{DR17} for a discussion of this computation.
When the empirical risk is near zero, the KL-generalization error is essentially the generalization error. 
 As empirical risk increases, the bound loosens and the square root of the KL-generalization error bounds the generalization error.

\subsection{Data-dependent priors}
\label{ddpsec}

The prior $P$ that appears in the PAC-Bayes generalization bound must be chosen independently of the data $S \sim \Dist^m$, 
but can depend on the data distribution $\Dist$ itself.
If a data-dependent prior $\PAlg(S)$
does not depend too much on any individual data point, it should behave as if it depends only on the distribution.
\cref{dpthm} allows us to formalize this intuition: we can obtain new PAC-Bayes bounds that use data-dependent priors, provided they are $\epsilon$-differentially private:
We provide an example using the bound of Maurer (\cref{pacbayes}).

\begin{theorem}[PAC-Bayes with private data-dependent priors]\label{DPpacbayes}
Fix a bounded loss $\loss \in [0,1]$.
Let $m \in \Nats$, 
let $\PAlg \colon Z^m \randto \ProbMeasures{\HH}$ be an $\epsilon$-differentially private mechanism for choosing a data-dependent prior,
let $\Dist \in \ProbMeasures{Z}$,
and let $S \sim \Dist^m$.
Then, 
with probability at least $1-\delta$,
\begin{equation}\label{DPpacbound}
      \forall Q \in \ProbMeasures{\HH},\ 
      \KLbin{\EmpRisk{S}{Q}}{\Risk{\Dist}{Q}} 
      \le \frac {\KL{Q}{\PAlg(S)} + \ln \frac{4 \sqrt{m}}{\delta} }{m} + \epsilon^2/2 + \epsilon \sqrt{\frac{\ln(4/\delta)}{2m}} . 
\end{equation}
\end{theorem}

See \cref{proofofDPpacbayes} for a proof of a more general statement of the theorem and further discussion.
The main innovation here is recognizing the potential to choose data-dependent priors using private mechanisms.
The hard work is done by \cref{dpthm}: obtaining differentially private versions of other PAC-Bayes bounds is straightforward. 

When one is choosing the privacy parameter, $\epsilon$,
there is a balance between 
minimizing the direct contributions of $\epsilon$ to the bound (forcing $\epsilon$ smaller) and 
minimizing the indirect contribution of $\epsilon$ through the KL term for posteriors Q that have low empirical risk (forcing $\epsilon$ larger). 
The optimal value for $\epsilon$ is often much less than one, which can be challenging to obtain. 
We discuss strategies for achieving the required privacy in later sections.

\section{Weak approximations to $\epsilon$-differentially private priors}

\cref{DPpacbayes} permits data-dependent priors that are chosen by $\epsilon$-differentially private mechanisms.
In this section, we discuss concrete families of priors and mechanisms for choosing among them in data-dependent ways. 
We also relax \cref{DPpacbayes} to allow non-private priors.

We apply our main result to non-private data-dependent Gaussian priors with a fixed covariance matrix. %
Thus, we choose only the mean $\ww_0 \in \HH$ privately in a data-dependent way.
We show that it suffices for the data-dependent mean to be merely close in 2-Wasserstein
distance to a private mean to yield a generalization bound.
(It is natural to consider also choosing a data-dependent covariance,
 but as it is argued below, the privacy budget
we have in applications to generalization is very small.)

Ideally, we would choose a mean vector $\ww_0$ that leads to a tight bound. 
A reasonable approach is to choose $\ww_0$ by approximately minimizing the empirical risk $\EmpRisk{S}{}$ or surrogate risk $\SurRisk{S}{}$,
subject to privacy constraints. A natural way to do this is via an exponential mechanism.
We pause to introduce some notation for Gibbs distributions:
For a measure $P$ on $\HH$ and measurable function $g : \HH \to \Reals$, let $P[g]$ denote the expectation $\int g(h) P(\dee h)$ and, provided $P[g] < \infty$, let $P_{g}$ denote the probability measure on $\HH$, absolutely continuous with respect to $P$, with Radon--Nikodym derivative 
$
\frac{\dee P_{g}}{\dee P}(h) = 
\frac{g(h)}{P[g]}.
$
A distribution of the form $\Gibbs{P}{- \tau g }$  is generally referred to as a Gibbs distribution with energy function $g$ and inverse temperature $\tau$.
In the special case where $P$ is a probability measure, 
we call $\Gibbs{P}{-\tau \SurRisk{S}{} }$ a ``Gibbs posterior''.

\begin{lemma}[{\citealt[Thm.~6]{ExpRelease07}}]\label{asdfexrelease}
Let $ q : Z^m \times \HH \to \Reals$ be measurable, let $P$ be a $\sigma$-finite measure on $\HH$,
let $\beta > 0$,
and assume $P[\exp \parens {-\beta\, q(S,\cdot)}] < \infty$ for all $S \in Z^m$.
Let $\Delta q \defas \sup_{S,S'} \sup_{\ww \in \HH} | q(S,\ww) - q(S',\ww) |$, 
where the first supremum ranges over pairs $S,S' \in Z^m$ that disagree on no more than one coordinate.
Let $\Alg : Z^m \randto \HH$, on input $S \in Z^m$,
output a sample from the Gibbs distribution $\Gibbs{P}{-\beta q(S,\cdot)}$.
Then $\Alg$ is $2\beta \Delta q$-differentially private.
\end{lemma}

The following result is a straightforward application of \cref{asdfexrelease} %
and essentially equivalent results have appeared in numerous studies of the differential privacy of Bayesian and Gibbs posteriors
\citep{mir2013differential, bassily2014differentially, dimitrakakis2014robust, Wang:2015,Minami16}.

\begin{corollary}\label{dpGibbspSample}
Let $\tau > 0$ and let $\SurRisk{S}{}$ denote the surrogate risk, taking values in an interval of length $\Delta$.
One sample from the Gibbs posterior $P_{\exp \parens{-\tau \SurRisk{S}{}}}$ 
is $\frac{2 \tau \Delta}{m}$-differentially private.
\end{corollary}

\subsection{Weak convergence yields valid PAC-Bayes bounds} \label{weakconvsec}

Even for small values of the inverse temperature, 
it is difficult to implement the exponential mechanism because sampling 
from Gibbs posteriors exactly is intractable. 
On the other hand, a number of algorithms exist for generating approximate samples from Gibbs posteriors.
If one can control the total-variation distance, one can obtain a bound like \cref{DPpacbayes}  by 
applying approximate max-information bounds for $(\epsilon,\delta)$-differential private mechanisms.
However, many algorithms do not control the total-variation distance to stationarity, or do so poorly.

The generalization properties of randomized classifiers are generally insensitive to small variations of the parameters, however,
and so it stands to reason that our data-dependent prior need only be itself close to an $\epsilon$-differentially private prior.
We formalize this intuition here by deriving bounds on $\KL{Q}{\PAlg(S)}$ in terms of a non-private data-dependent prior $P^S$.
We start with an identity:

\begin{lemma}\label{KLlemma}
If $P' \ll P$ then
$\KL{Q}{P} = \KL{Q}{P'} + Q \brackets{\ln \rnd{P'}{P} }$.
\end{lemma}

The proof is straightforward. The lemma highlights the role of $Q$ in judging the difference between $P'$ and $P$, and leads immediately to the following corollary (see \cref{proofofweakconv}).

\begin{lemma}[Non-private priors]
\label{nonprivate}
Let $m \in \Nats$, 
let $\PAlg \colon Z^m \randto \ProbMeasures{\HH}$ be $\epsilon$-differentially private,
let $\Dist \in \ProbMeasures{Z}$, 
let $S \sim \Dist^m$,
and let $P^S$ be a data-dependent prior such that, 
for some $P^*(S)$ satisfying $\Pr[P^*(S)|S] = \Pr[\PAlg(S)|S]$, we have
$P^S \ll P^*(S)$ with probability at least $1-\delta'$.
Then, with probability at least $1-\delta-\delta'$,
\cref{DPpacbound} holds with $\KL{Q}{\PAlg(S)}$ replaced by $\KL{Q}{P^S} + Q \brackets{\ln \rnd{P^S}{P^*(S)} } $.
\end{lemma}

The conditions that $P^S \ll P^*(S)$ and
$
Q \brackets{\ln \rnd{P^S}{P^*(S)} }  \leq \infty
$
for some $Q$ 
do not constrain $P^S$ to be differentially private. 
In fact, $S$ could be $P^S$-measurable!

\cref{nonprivate} is not immediately applicable because $P^*(S)$ is intractable to generate.
The following application considers multivariate Gaussian priors,
$N(w)$, indexed by their mean vectors $w \in \HH$, 
 with a fixed positive definite covariance matrix $\Sigma \in \Reals^{p \times p}$.
We require two technical results: 
The first implies that we
can bound $Q \brackets{\ln \rnd{P^S}{P^*(S)} }$
if $Q$ concentrates near the non-private mean;
The second characterizes this concentration for Gibbs posteriors built from bounded surrogate risks.
\begin{lemma}\label{normboundRN}
$
Q \brackets {\ln \rnd{N(w')}{N(w)}} 
\le   \frac 1 2 \norm{w-w'}^2_{\SigmaInv} +  \norm{w-w'}_{\SigmaInv} \, \, \EEE{v \sim Q}\norm{v-w'}_{\SigmaInv} .
$
\end{lemma}

\begin{lemma}
\label{GibbsboundRN}
Let $P=N(w)$ and $Q = \Gibbs{P}{h}$ for $h \ge 0$.
Then
$
\EEE{v \sim Q}\norm{v-w}_{\SigmaInv}  \le  \sqrt{2 \smash{\Lpnorm{\infty}{P}{h}} } + \sqrt{2/\pi}.
$
\end{lemma}

\begin{corollary}[Gaussian means close to private means]\label{twocontrol}
Let $m \in \Nats$, 
let $\Dist \in \ProbMeasures{Z}$, 
let $S \sim \Dist^m$,
let $\Alg : Z^m \randto \HH$ be $\epsilon$-differentially private,
and let $w(S)$ denote a data-dependent mean vector such that, 
for some $w^*(S)$ satisfying $\Pr[w^*(S)|S] = \Pr[\Alg(S)|S]$,
we have 
\begin{equation}\label{meanbound}
\norm{w(S) - w^*(S)}^2_2 \le C
\end{equation}
with probability at least $1-\delta'$.
Let $\mineig$ be the minimum eigenvalue of $\Sigma$.
Then, with probability at least $1- \delta - \delta'$,
\cref{DPpacbound} holds
with $\KL{Q}{\PAlg(S)}$ replaced by 
$\KL{Q}{N(w(S))} +  \frac 1 2 \, {C}/{\mineig}  + \sqrt{{C}/{\mineig}} \, \EE_{v \sim Q}\norm{v-w(S)}_{\SigmaInv}$.
In particular, for a Gibbs posterior $Q = \Gibbs{P^S}{-\tau \smash{\SurRisk{S}{}}}$, we have
$\EE_{v \sim Q}\norm{v-w(S)}_{\SigmaInv} \le \sqrt{2\tau\Delta} + \sqrt{2/\pi}$.
\end{corollary}

See \cref{proofofweakconv} for details and further discussion.

One way to achieve \cref{meanbound} is to construct $w_1(S),w_2(S),\dots$ so that 
$\Pr[w_n(S)|S] \to \Pr[\Alg(S)|S]$ weakly with high probability. Skorohod's representation theorem then implies the existence of $w^*(S)$.
One of the standard algorithms used to generate such sequences for high-dimensional Gibbs distributions is
stochastic gradient Langevin dynamics \citep[SGLD;][]{welling2011bayesian}.

In order to get nonasymptotic results, 
it suffices to 
bound the 2-Wasserstein distance of the SGLD Markov chain to stationarity.
Recall that the $p$-Wasserstein distance between $\mu$ and $\nu$ is
given by 
$
(\pWD{p}{\mu}{\nu})^p = \inf_\gamma \int \norm{v-w}_2^p \, \dee\gamma(v,w)
$
where the infimum runs over couplings of $\mu$ and $\nu$,
i.e., 
distributions $\gamma \in \ProbMeasures{\HH \times \HH}$ with marginals $\mu$ and $\nu$, respectively.

Let $\Alg(S)$ return a sample from the Gibbs posterior $\Gibbs{P}{-\tau\smash{\SurRisk{S}{}}}$ with Gaussian $P$
and surrogate risk $\smash{\SurRisk{S}{}}$ constructed from smooth loss functions taking values in a length-$\Delta$ interval. 
By \cref{dpGibbspSample}, $\Alg$ is $\frac{2 \tau \Delta}{m}$-differentially private.
Consider running SGLD to target $\Pr[\Alg(S)|S] = \Gibbs{P}{-\tau\smash{\SurRisk{S}{}}}$.
Assume\footnote{
The status of this assumption relates to results by \citet[Thm.~7.3]{MATTINGLY2002185} and \citet[Prop.~3.3]{RRT17},
under so-called dissipativity conditions on the regularized loss.
We have hidden potentially exponential dependence on $\tau$, which is problem dependent. 
See also \citep{EMS18}.
} that, for every $c > 0$, there is a step size $\eta > 0$ and number of SGLD iterations $n \in O(\frac 1 {c^q})$,
such that the $n$-th iterate $w(S) \in \HH$ produced by SGLD satisfies
$
\pWD [\big]{2}{ \Pr[w(S)|S] }{ \Pr[\Alg(S)|S]  } \le c.
$
Markov's inequality and the definition of $\ppWD_{2}$ immediately implies the following.

\begin{corollary}[Prior via SGLD]\label{sgldgen}
Under the above assumption, for some step size $\eta > 0$, 
the $n$-th iterate $w(S)$ of 
SGLD targeting 
$\Gibbs{P}{-\tau\smash{\SurRisk{S}{}}}$
satisfies \cref{twocontrol} with $\epsilon = \frac{2 \tau \Delta}{m}$ and $C \in O(\frac 1 {\delta' n^{1/q}})$.
\end{corollary}

The dependence on $\delta'$ is poor. 
However, one can construct Markov chain algorithms that are geometrically ergodic, in which case $n^{-1/q}$ is replaced by a term $2^{-\Omega(n)}$,
allowing one to spend computation to control the $1/\delta'$ term.

\section{Empirical studies}

We have presented data-dependent bounds and so it is necessary to study them empirically to evaluate their usefulness. 
The goal of this section is to make a number of arguments.
First,
 it is an empirical question as to what value of the \invtemp{} $\tau$ is sufficient to yield small empirical risk from a Gibbs posterior. 
 Indeed, we compare to the bound  of \citet{LEVER2013}, presented above as \cref{leverbound}. As \citeauthor{LEVER2013} point out, this bound depends explicitly on $\tau$, where it plays an obvious role as a measure of complexity.
 Second, despite how tight this bound is for small values of $\tau$, 
 the bound must become vacuous before the Gibbs posterior \emph{would have} started to overfit on random labels
 because this bound holds for \emph{all} data distributions.
 We demonstrate that this phase transition happens well before the Gibbs posterior achieves its minimum risk on true labels.
 Third, because our bound retains the KL term, we can potentially identify easy data. 
 Indeed, our risk bound decreases beyond the point where the same classifier begins to overfit to random labels.
 Finally, our results suggest that we can use the property that SGLD converges weakly to investigate the generalization properties of Gibbs classifiers.  More work is clearly needed to scale our study to full fledged deep learning benchmarks. 
 (See \cref{KLeval} for details on how we compute the KL term and challenges there due to Gibbs posteriors $Q$ not be exactly samplable.)%

More concretely, we perform an empirical evaluation using SGLD to approximate simulating from Gibbs distributions. 
\cref{twocontrol} provides some justification by showing that a slight weakening of the PAC-Bayes generalization bound is valid provided that SGLD eventually controls the 2-Wasserstein distance to stationarity. 
However, because we cannot measure our convergence in practice, it is an empirical question as to whether our samples are accurate enough. 

Violated bounds would be an obvious sign of trouble. We expect the bound on the classification error not to go below the true error as estimated on the heldout test set (with high probability).
We perform an experiment on a MNIST (and CIFAR10, with the same conclusion so we have not included it) using 
true and random labels and find that no bounds are violated. The results suggest that it may be useful to empirical study bounds for Gibbs classifiers using SGLD.

Our main focus is a sythentic experiment comparing the bounds of \citet{LEVER2013} to our new bounds based on privacy.
The main finding here is that, as expected, the bounds by \citeauthor{LEVER2013} must explode when the Gibbs classifier begins to overfit random labels, whereas our bounds, on true labels, continue to track the training error and bound the test error.

\subsection{Setup}

Our focus is on classification 
by neural networks into $K$ classes. Thus $Z= X \times [K]$, and 
we use neural networks that output probability vectors over these $K$ classes.
Given weights $\ww \in \HH$ and input $x \in X$, the probability vector output by the network 
is $p(\ww,x) \in [0,1]^K$.
Networks are trained by minimizing cross entropy loss:
$\loss(\ww,(x,y)) = g(p(\ww,x),y)$, where
$g((p_1,\dots,p_K),y) = -\ln p_{y}$. 
Note that cross entropy loss is merely bounded below. 
We report results in terms of $\set{0,1}$-valued classification error:
$\loss(\ww,(x,y)) = 0$ if and only if $y$ is the largest coordinate of $p(\ww,x)$.

We refer to elements of $\HH$ and $\RHH$ as classifiers and randomized classifiers, respectively, and to the (empirical) 0--1 risk as the (empirical) error.
We train two different architectures using SGLD on MNIST and a synthetic dataset, SYNTH. The experimental setup is explained in \cref{app:setup}.

\paragraph{One-stage training procedure}

We run SGLD for $T$ training epochs with a fixed value of the parameter $\tau$.
We observe that convergence appears to occur within 10 epochs, but
use a much larger number of training epochs to potentially expose nonconvergence behavior.
The value of the \invtemp{} $\tau$ is fixed during the whole training procedure. 

\paragraph{Two-stage training procedure}

In order to evaluate our private PAC-Bayes bound (\cref{DPpacbayes}), we perform a two-stage training procedure: 
\begin{itemize}
	\item \textbf{Stage One.} We run SGLD for $T_1$ epochs with \invtemp{} $\tau_1$, minimizing the standard cross entropy objective. Let $w_0$ denote the neural network weights after stage one.
	\item \textbf{Transition.} We restart the learning rate schedule and continue SGLD for $T_1$  epochs with linearly annealing the temperature between $\tau_1$ and $\tau_2$, i.e.,
		 \invtemp{} $\tau_t = ((t-T_1)\tau_2 + (2T_1-t)* \tau_1)/T_1$, where $t$ is the current epoch number. 
	The objective at $\ww$ is the cross entropy loss for $\ww$ plus a weight-decay term $\frac{\gamma}{2}\| \ww- \ww_0\|_2^2$. 	 
	\item \textbf{Stage Two.} We continue SGLD for $T_2-T_1$ epochs with \invtemp{} $\tau_2$. 
	The objective is the same as in the transition stage.
\end{itemize}	

During the first stage, the $k$-step transitions of SGLD converge weakly towards a Gibbs distribution with a uniform base measure, producing a random vector $\ww_0 \in \HH$.
The private data-dependent prior $P_{\ww_0}$ is the Gaussian distribution centred at $\ww_0$ with diagonal covariance $\frac{1}{\gamma} I_{p}$.
During the second stage,
SGLD converges to the Gibbs posterior with a Gaussian base measure $P_{\ww_0}$, 
i.e., $\Qt{\tau_2} = P_{\exp \parens{ -\tau_2 \EmpRisk{S}{} }}$.

\paragraph{Bound calculation}

Our experiments evaluate different values of the \invtemp{} $\tau$.

We evaluate Lever bounds for the randomized classifier $\Qt{\tau}$ obtained by the one-stage training procedure, with $T=1000$. 
We do so on both the MNIST and \syn{} datasets.

We also evaluate our private PAC-Bayes bound (\cref{DPpacbayes}) for the randomized classifier $\Qt{\tau_2}$ and the private data-dependent prior $P_{\ww_0}$, where the privacy parameter depends on $\tau_1$. 
The bound depends on the value of the $\KL{\Qt{\tau_1}}{P_{\ww_0}}$. 
The challenges of estimating this term are described in \cref{KLeval}. 
We only evaluate the differentially private PAC-Bayes bounds on the small neural network and \syn{} dataset,

The parameter settings for SYNTH experiments are: $T1=100$, $T2=1000$, $\gamma=2$; for MNIST: $T1=500$, $T2=1000$, $\gamma=5$. When evaluating Lever bounds with a one-stage learning procedure for either datasets, $T=1000$. 

\subsection{Results}

\begin{figure}[t]
\centering
\includegraphics[width=.35\linewidth]{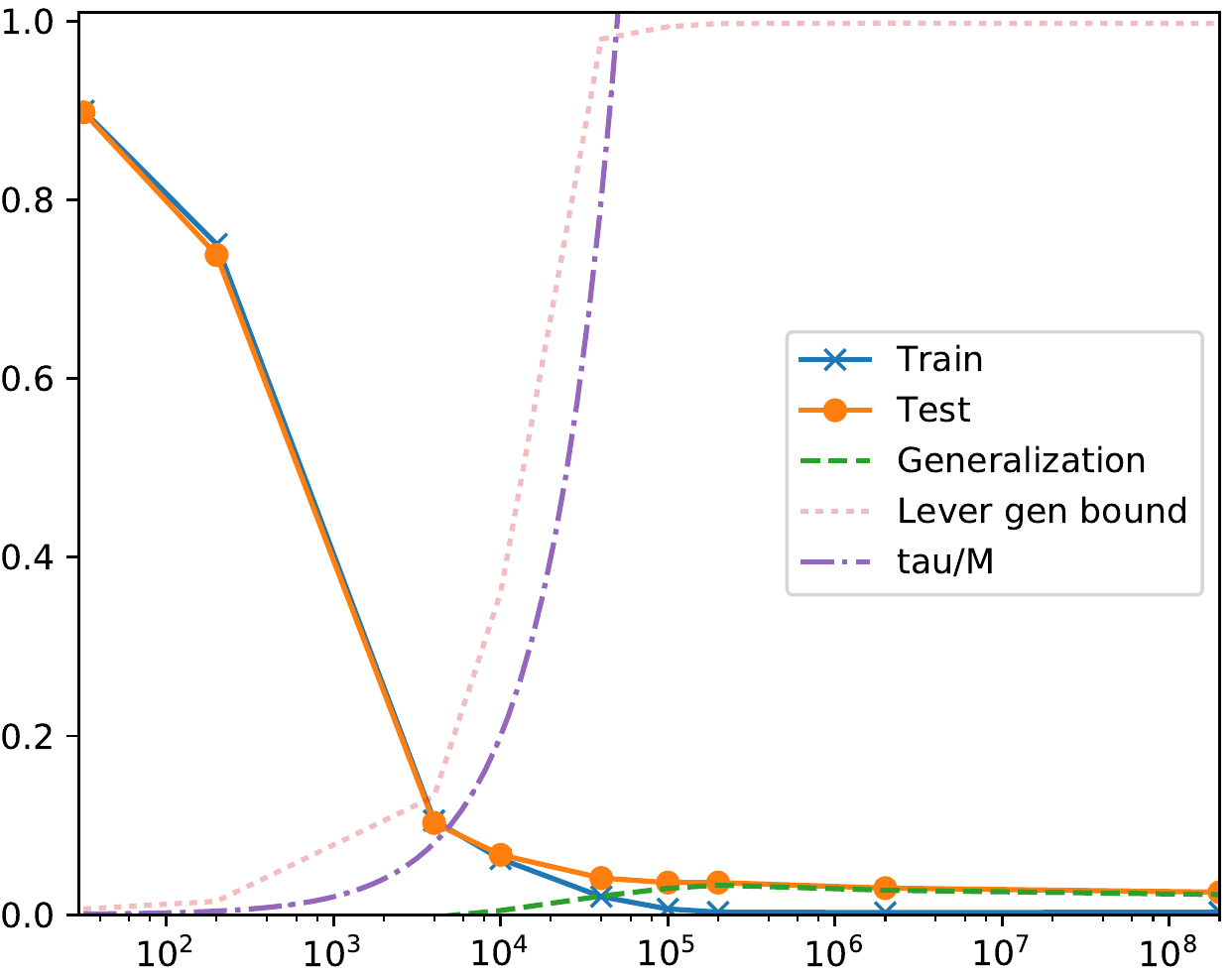}\ 
\includegraphics[width=.35\linewidth]{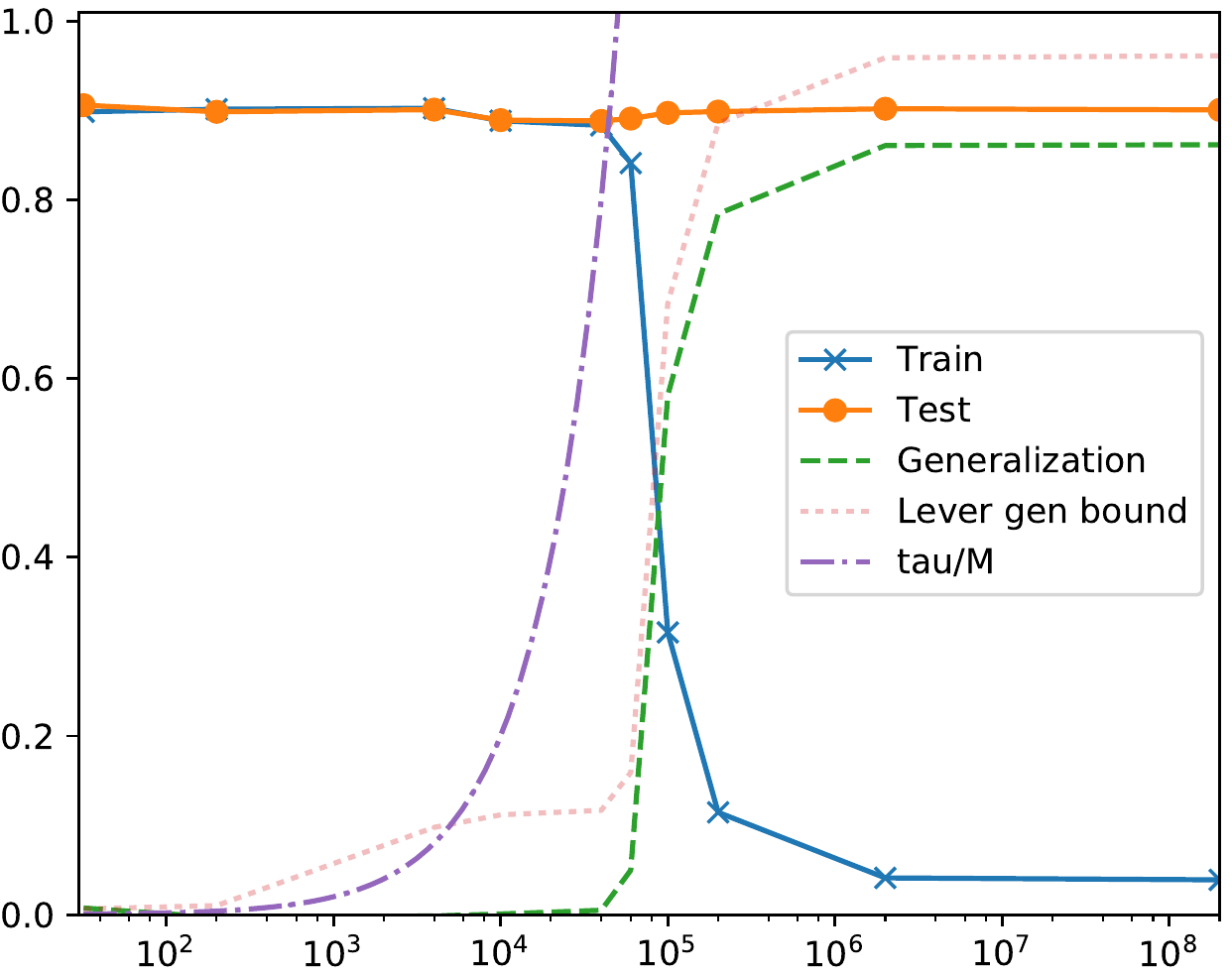}
\caption{Results for a fully connected neural network trained on MNIST dataset with SGLD and a fixed value of $\tau$. We vary $\tau$ on the x-axis. The y-axis shows the average 0--1 loss. We plot the estimated generalization gap, which is the difference between the training and test errors. 
The left plot shows the results for the true label dataset. We observe, that the training error converges to zero as $\tau$ increases. Further, while the generalization error increases for intermediate values of $\tau$ ($10^{4}$ to $10^{6}$), it starts dropping again as $\tau$ increases. We see that the Lever bound fails to capture this behaviour due to the monotonic increase with $\tau$. 
The right hand side plot shows the results for a classifier trained on random labelling of MNIST images. The true error is around $0.9$. For small values of $\tau$ (under $10^3$) the network fails to learn and the training error stays at around $0.9$. When $\tau$ exceeds the number of training points, the network starts to overfit heavily. The sharp increase of the generalization gap is predicted by the Lever bound.
}
\label{MNISTfig}
\end{figure}

\begin{figure}[t]
\centering
\includegraphics[width=.325\linewidth]{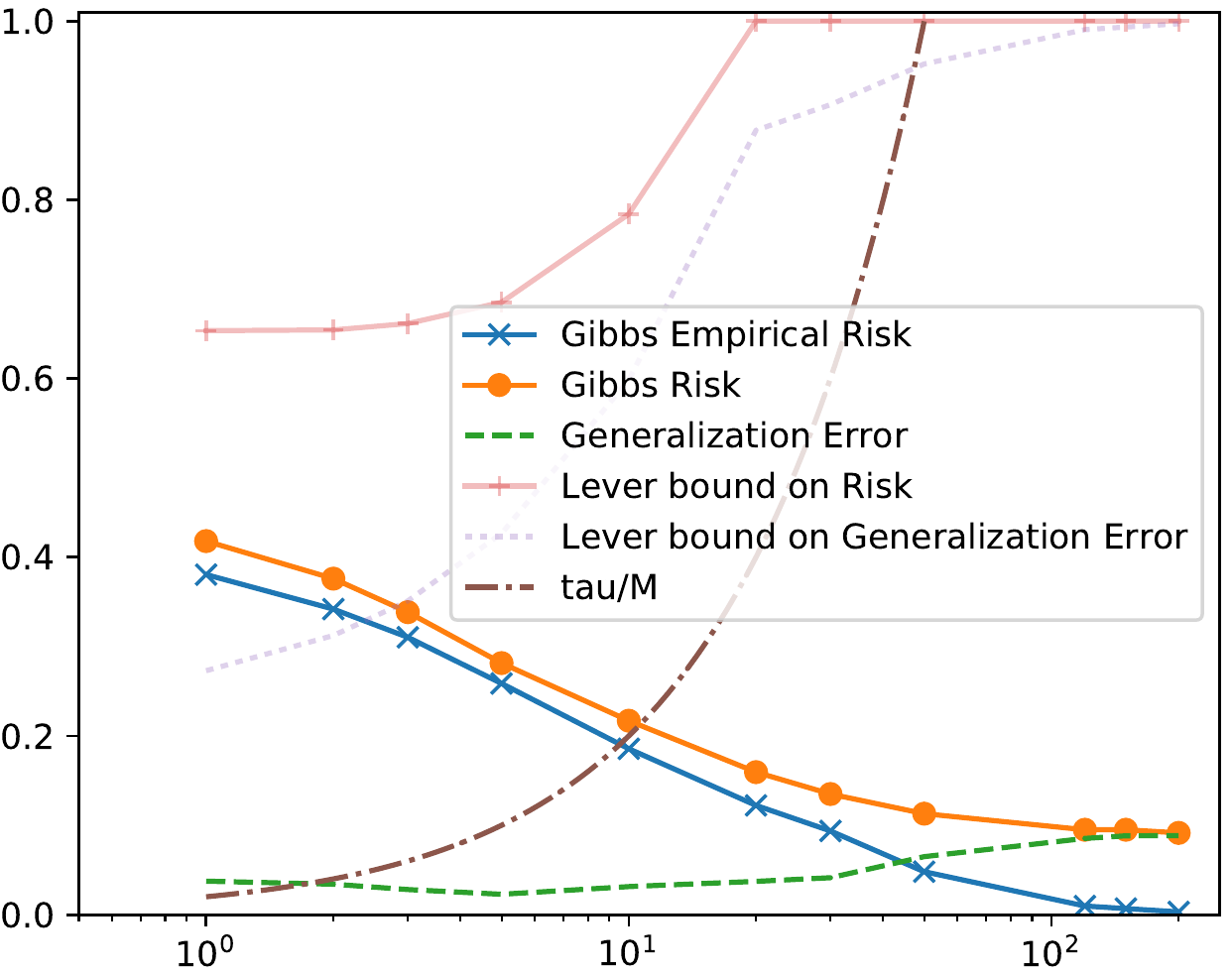}%
\includegraphics[width=.325\linewidth]{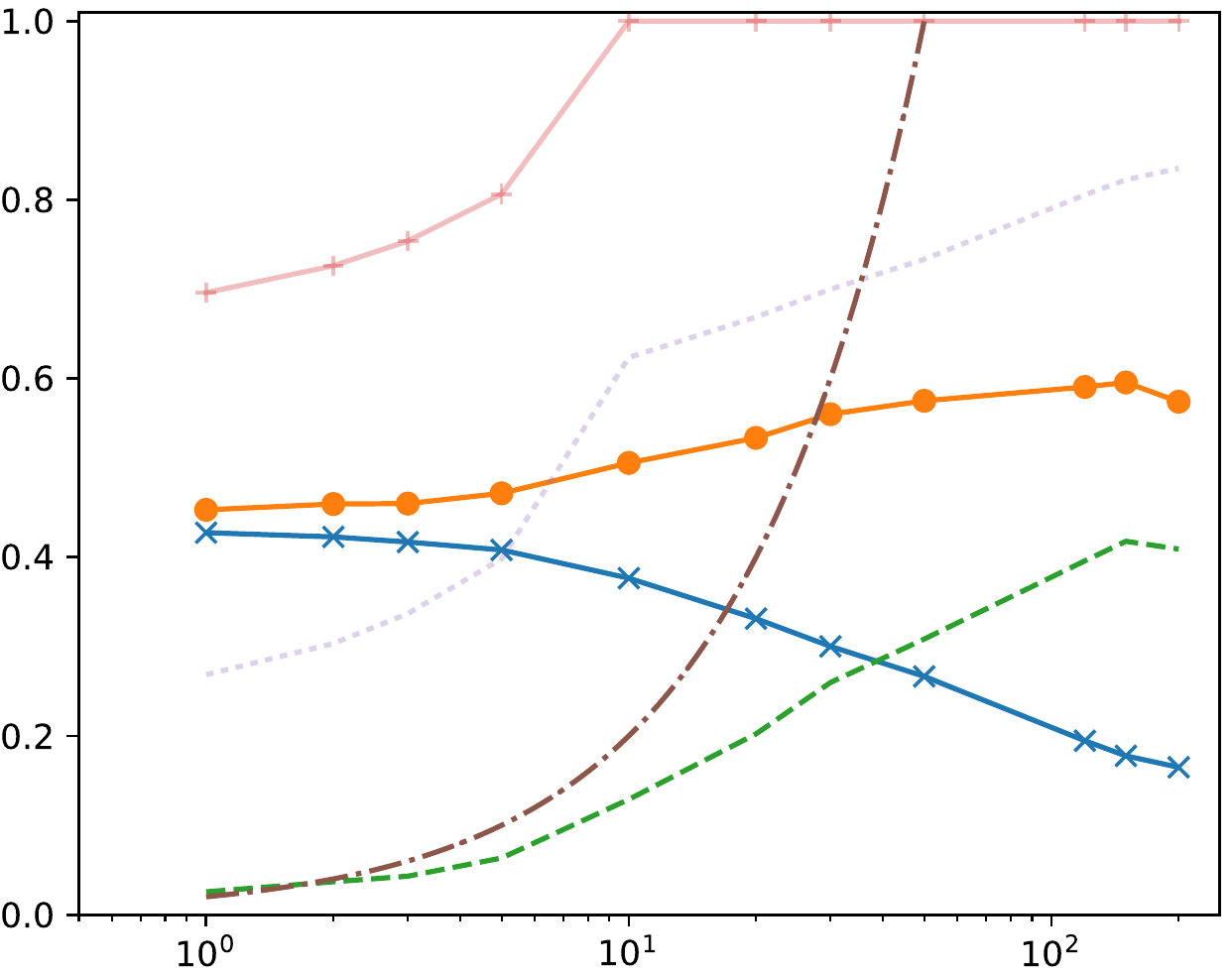}%
\includegraphics[width=.35\linewidth]{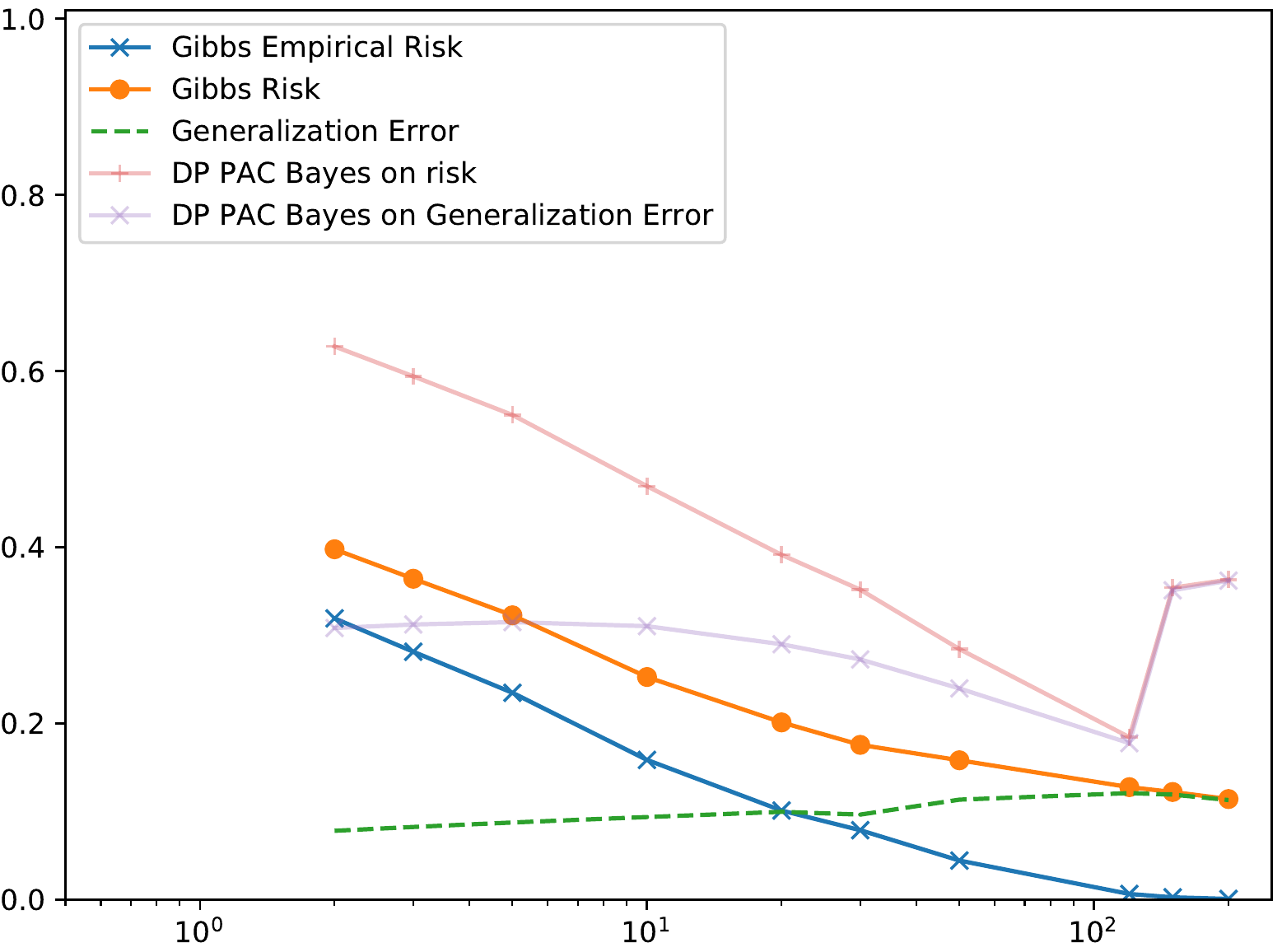}
\caption{
Results for a small fully connected neural network trained on a synthetically generated dataset \syn, consisting of 50 training examples. 
The x-axis shows the $\tau$ value, and the y-axis the average 0--1 loss. To generate the top plots, we train the network with a one-stage SGLD. The top-left plot corresponds to the true label dataset, and the top-right to the random label dataset. Similarly as in MNIST experiments, we do not witness any violation of the Lever bounds. Once again, we notice that Lever bound gets very loose for larger values of $\tau$ in the true label case. 
The bottom plot demonstrates the results for the two-stage SGLD. In this case the x-axis plots the $\tau$ value used in the second-stage optimization. The first stage used $\tau_1=1$. The network is trained on true labels. We see that the differentially private PAC-Bayes bound yields a much tighter estimate of the generalization gap for larger values of $\tau$ than the Lever bound (top left). When $\tau$ becomes very large relative to the amount of training data, it becomes more difficult to sample from the Gibbs posterior. This results in a looser upper bound on the KL divergence between the prior and posterior.
}
\label{synth}
\end{figure}

Results are presented in \cref{MNISTfig,synth}. 
We never observe a violation of the PAC-Bayes bounds for Gibbs distributions. This suggests that 
our assumption that SGLD has nearly converged is accurate enough or the bounds are sufficiently loose that any effect from nonconvergence was masked. 

Our MNIST experiments highlight that the Lever bounds upper bound the risk for \emph{every} possible data distribution, including the random label distribution.
In the random label experiment (\cref{MNISTfig}, right plot), 
when $\tau$ gets close to the number of training samples, 
the generalization error starts increasing steeply.
This phase transition is captured by the Lever bound. 
In the true label experiment (right plot), the generalization error does not rise with $\tau$. Indeed,
it continues to decrease, and so 
the Lever bound quickly becomes vacuous as we increase $\tau$.
The Lever bound cannot capture this behavior because it must simultaneously bound the generalization error under random labels.

On the \syn{} dataset, we see the same phase transition under random labels and so Lever bounds remain vacuous after this point. In contrast, we see that our private PAC-Bayes bounds can track the error beyond the phase transition that occurs under random labels. (See \cref{synth}.) At high values of $\tau$, 
our KL upper bound becomes very loose.

\paragraph{Private versus Lever PAC-Bayes bound}

While Lever PAC-Bayes bound fails to explain generalization for high $\tau$ values, our private PAC-Bayes bound may remain nonvacuous. 
This is due to the fact that it retains the KL term, which is sensitive to the data distribution via $Q$, and thus it can be much lower than the upper bound on the KL in \citet{LEVER2013}  for datasets with small true Bayes risk. %
Two stage optimization, inspired by the DP PAC-Bayes bound, allows us to obtain more accurate classifiers  by setting a higher  \invtemp{} parameter at the second stage, $\tau_2$.

We do not plot DP PAC-Bayes bound for MNIST experiments due to the computational challenges approximating the KL term for a high-dimensional parameter space, as discussed in  \cref{KLeval}. 
We evaluate our private PAC-Bayes bound on MNIST dataset only for a combination of $\tau_1=10^3$ and $\tau_2\in [3*10^3, 3*10^{4}, 10^5, 3*10^5 ] $. 
The values are chosen such that $\tau_1$ gives a small penalty for using the data to learn the prior and $\tau_2$ is chosen such that
at $\tau = \tau_2$ Lever's bound returns a vacuous bound (as seen in \cref{MNISTfig}). We use $10^5$ number of samples from the DP Gaussian prior learnt in stage one to approximate $\ln \Zt{}$ term that appears in the KL, as defined in \cref{logzdefn}.

The results are presented in the table below.
 While DP PAC-Bayes bound is very loose, it is still smaller than Lever's bound for high values of \invtemp{}.
 
 Note, that for smaller values $\tau_2$, we can use Lever's upper bound on the KL term instead of performing a Monte Carlo approximation. Since $\tau_1$ is small and adds only a small penalty ($\sim 1\%$), the DP PAC-Bayes bound is equal to Lever's bound plus a differential privacy penalty ($\sim 1\%$).
 
\begin{center}
\begin{tabular}{c | c c c c}
   $\tau_2$ & $3*10^3$ & $3*10^{4}$ & $10^5$ & $3*10^5$ \\
   \hline
  Test & 0.12 & 0.07  & 0.06 & 4 \\
  DP PAC-Bayes bound on test & 0.21 & 0.35 & 0.65 & 1 \\
  Lever PAC-Bayes with $\tau = \tau_2$ on test & 0.26 & 1 & 1 & 1 \\
\end{tabular}
\end{center}

\newpage
\subsection*{Acknowledgments}
The authors would like to thank 
Olivier Catoni,
Pascal Germain,
Mufan Li,
David McAllester,
and
Alexander Rakhlin,
John Shawe-Taylor,
for helpful discussions.
This research was carried out in part while the authors were visiting the Simons Institute for the Theory of Computing at UC Berkeley.
GKD was additionally supported by an EPSRC studentship.  
DMR was additionally supported by an
NSERC Discovery Grant and Ontario Early Researcher Award.

\printbibliography%

\newpage
\appendix

\textbf{Supplementary Material for ``Data-dependent PAC-Bayes priors via differential privacy''} \\
See \url{https://arxiv.org/abs/1802.09583} for the full paper.

\section{Basic Differential Privacy}
\label{DPintro}

See \citep{Dwork2006,dwork2014algorithmic} for more details.

Let $U,U_1,U_2,\dots$ be independent uniform $(0,1)$ random variables, 
independent also of any other random variables unless stated otherwise,
and let $\pi : \Nats \times [0,1] \to [0,1]$ satisfy $(\pi(1,U),\dots,\pi(k,U)) \eqdist (U_1,\dots,U_k)$ for all $k \in \Nats$. 
Write $\pi_k$ for $\pi(k,\cdot)$.

\begin{definition}
Let $R,T$ be measurable spaces. A \defn{randomized algorithm} $\Alg$ from $R$ to $T$, denoted $\Alg\colon R \randto T$,
is a measurable map $\Alg\colon [0,1] \times R \to T$.
Associated to $\Alg$ is a (measurable) collection of random variables $\{ \Alg_r : r \in R \}$ that satisfy
$\Alg_r = \Alg(U,r)$. 
When there is no risk of confusion, we write $\Alg(r)$ for $\Alg_r$.
\end{definition}

For our purposes, we rely only on the fact that privacy is preserved under post-processing, which we now define.\footnote{It is sometimes more natural to refer to the differential privacy of probability kernels, i.e.,
measurable maps from $Z^m$ to $\ProbMeasures{T}$, and to $S$-measurable 
random probability measures $Q$, i.e., probability kernels defined on the basic probability space satisfying
$Q = g(S)$ for some probability kernel $g : Z^m \to \ProbMeasures{T}$, where $S \sim \Dist^m$.
In both cases, the connection to the above definition is the same: for every probability kernel $\kappa : R \to \ProbMeasures{T}$
there exists $\Alg \colon [0,1] \times R \to T$ such that $\Alg(U,r) \sim \kappa(r)$ for every $r \in R$.
In the other direction, clearly $\kappa(r)(A) = \Pr \set{ \Alg(U,r) \in A}$ for every measurable $A \subseteq T$.}

\begin{definition}
Let $\Alg \colon R \randto T$ and $\Alg' \colon T \randto T'$.
The \defn{composition} $\Alg' \circ \Alg \colon R \randto T'$ is given by
$(\Alg' \circ \Alg)(u,r) = \Alg'(\pi_2(u),\Alg(\pi_1(u),r))$.
\end{definition}

\begin{lemma}[post-processing]\label{dppost}
Let $\Alg\colon Z^m \randto T$ be $(\epsilon,\delta)$-differentially private
and let $F \colon T \randto T'$ be arbitrary.
Then $F \circ \Alg$ is $(\epsilon,\delta)$-differentially private.
\end{lemma}

\section{Proof of \cref{DPpacbayes}}
\label{proofofDPpacbayes}

We prove a slightly more general result.

\begin{theorem}\label{genDPpacbayes}
Fix a bounded loss $\loss \in [0,1]$.
Let $m \in \Nats$, 
let $\PAlg \colon Z^m \randto \ProbMeasures{\HH}$ be an $\epsilon$-differentially private data-dependent prior,
let $\Dist \in \ProbMeasures{Z}$,
and let $S \sim \Dist^m$.
Then, for all $\delta \in (0, 1)$ and $\beta \in (0,\delta)$, 
with probability at least $1-\delta$,
\begin{equation}\label{genDPpacbound}
      \forall Q \in \ProbMeasures{\HH},\ 
      \KLbin{\EmpRisk{S}{Q}}{\Risk{\Dist}{Q}} \le \frac {\KL{Q}{\PAlg(S)} + \ln \frac{2 \sqrt{m}}{\delta-\beta} }{m} + \epsilon^2/2 + \epsilon \sqrt{\frac{\ln(2/\beta)}{2m}} . 
\end{equation}
\end{theorem}

\begin{proof}
For every distribution $P$ on $\HH$,
let
\begin{equation}
R(P) = 
 \set[\Big]{ 
         S \in Z^m : 
        (\exists Q)\ 
            \KLbin{\EmpRisk{S}{Q}}{\Risk{\Dist}{Q}} \ge  m^{-1}\parens[\big]{\KL{Q}{P} + \ln \frac{2\sqrt{m}}{\delta'}  }
       }.
\end{equation}
It follows from \cref{pacbayes} that
$
\Pr_{S \sim \Dist^m}  \event{S \in R(P)} \le \delta'.
$
Let $\beta > 0$.
Then, by the definition of approximate max-information, we have
\[
\Pr_{S \sim \Dist^m}  \event{S \in R(\PAlg(S))} 
&\le \e^{\amaxinf{\beta}{\PAlg}{m}} \PPr{(S,S') \sim \Dist^{2m}}  \event{S \in R(\PAlg(S'))} + \beta \\
&\le \e^{\amaxinf{\beta}{\PAlg}{m}} \delta'  + \beta \defas \delta.
\]
We have $\delta' = \e^{-\amaxinf{\beta}{\PAlg}{m}}(\delta - \beta)$. 
Therefore,
with probability no more than $\delta$ over $S \sim \Dist^{m}$, 
\begin{equation} \label{DPeasy}
     \exists Q \in \ProbMeasures{\HH},\ 
      \KLbin{\EmpRisk{S}{Q}}{\Risk{\Dist}{Q}} 
      \ge \frac {\KL{Q}{\PAlg(S)} + \ln \frac{2\sqrt{m}}{\delta - \beta} + \amaxinf{\beta}{\PAlg}{m}}{m}.
\end{equation}
The result follows from replacing the approximate max-information $\amaxinf{\beta}{\PAlg}{m}$ 
with the bound provided by \cref{dpthm}.
\end{proof}

The theorem leaves open the choice of $\beta < \delta$.
For any fixed values for $\epsilon$, $m$, and $\delta$, it is easy to optimize $\beta$ to obtain the tighest possible bound.
In practice, however, the optimal bound is almost indistinguishable from that obtained by taking $\beta = \delta / 2$. 
For the remainder of the paper, we take this value for $\beta$, 
in which case, the r.h.s.\ of \cref{DPpacbound} is
\[
\frac {\KL{Q}{\PAlg(S)} + \ln \frac{4 \sqrt{m}}{\delta} }{m} + \epsilon^2/2 + \epsilon \sqrt{\frac{\ln(4/\delta)}{2m}} . 
\]

Note that the bound holds for all posteriors $Q$. In general the bounds are interesting 
only when $Q$ is data dependent, otherwise one can obtain tighter bounds
via concentration of measure results for empirical means of bounded i.i.d.~random variables.

When one is choosing the privacy parameter, $\epsilon$,
there is a balance between 
minimizing the direct contributions of $\epsilon$ to the bound (forcing $\epsilon$ smaller) and 
minimizing the indirect contribution of $\epsilon$ through the KL term for posteriors Q that have low empirical risk (forcing $\epsilon$ larger). 
One approach is to  compute the value of $\epsilon$ that achieves a certain bound on the excess generalization error.  In particular,
choosing $\epsilon^2/2 = \alpha$ contributes an additional gap of $\alpha$ to the KL-generalization error. 
Choosing $\alpha$ is complicated by the fact that there is a non-linear relationship between the generalization error and the KL-generalization error, depending on the empirical risk.
A better approach is often to attempt to balance the direct contribution with the indirect one.
Regardless, the optimal value for $\epsilon$ is much less than one, which can be challenging to obtain. 
We discuss strategies for achieving the required privacy in later sections.

\section{Proofs for \cref{weakconvsec}}
\label{proofofweakconv}

These results connect approximations to differential privacy with bounds on the KL term in a PAC-Bayesian bound, yielding PAC-Bayes bounds that hold even if the prior is chosen via a nonprivate mechanism.
In independent work, subsequent to our original arXiv preprint, \citet{NIPS2018_8134} combined PAC-Bayesian bounds and stability to leverage \emph{distribution} dependent priors. Their approach is distinct, though complimentary.
See also work by \cite{NIPS2017_6886}, who combines stability and PAC-Bayesian bounds in yet another way.

\begin{proof}[Proof of \cref{KLlemma}]
Assume $Q \ll P'$, for otherwise the bound is trivial as $\KL{Q}{P'} = \infty$. 
Then $Q \ll P$, because $P' \ll P$, and so
$
\rnd{Q}{P} = \rnd{Q}{P'} \rnd{P'}{P}
$
and
\[
\KL{Q}{P} - \KL{Q}{P'}
&= Q \brackets[\Big]{ \ln \rnd{Q}{P} } - \KL{Q}{P'} 
\\&= Q \brackets[\Big]{ \ln \rnd{Q}{P'} + \ln \rnd{P'}{P} } - \KL{Q}{P'}
\\&= \KL{Q}{P'} + Q \brackets[\Big]{\ln \rnd{P'}{P} } - \KL{Q}{P'}  
\\&= Q \brackets[\Big]{\ln \rnd{P'}{P} }.
\]
\end{proof}

\begin{proof}[Proof of \cref{nonprivate}]
Let $P^*(S)$ satisfy the conditions in the statement of the theorem. 
Then $P^*(S)$ is $\epsilon$-differentially private.
By \cref{DPpacbayes}, the bound in \cref{DPpacbound} holds with probability at least $1- \delta$
for the data-dependent prior $P^*(S)$ and all posteriors $Q$. 
By hypothesis, with probability $1-\delta-\delta'$, $P^S \ll P^*(S)$,
and so, by \cref{KLlemma},
$
\KL{Q}{P^*(S)} = \KL{Q}{P^S} + Q \brackets{\ln \rnd{P^S}{P^*(S)} }.
$
\end{proof}

\begin{proof}[Proof of \cref{normboundRN}]
Expanding the log ratio of Gaussian densities  and then applying Cauchy--Schwarz, we obtain
\begin{align}
 {\ln \rnd{N(w')}{N(w)}}  (v)
&= \frac 1 2  \parens[\big]{\norm{w - v}^2_{\SigmaInv} - \norm{w'-v}^2_{\SigmaInv} }  \\
&=  \ip{w'-w}{v}_{\SigmaInv} + \frac 1 2 \left(\norm{w}^2_{\SigmaInv} -\norm{w'}^2_{\SigmaInv} \right) \\
&= \frac 1 2 \ip{w'-w}{2v}_{\SigmaInv} - \frac 1 2 \ip{w'-w}{w+w'}_{\SigmaInv} \\
&= \frac 1 2 \ip{w'-w}{2v - w - 2w' + w'}_{\SigmaInv}   \\
&= \frac 1 2  \ip{w'-w}{2(v-w') + w'-w}_{\SigmaInv}  \\
&\le \frac 1 2 \norm{w'-w}^2_{\SigmaInv} +  \norm{w'-w}_{\SigmaInv} \, \, \norm{v-w'}_{\SigmaInv}. 
\end{align}
The result follows by taking the expectation with respect to $v \sim Q$.
\end{proof}
\begin{proof}[Proof of \cref{GibbsboundRN}]
Let $g = \rnd{Q}{P} \defas \frac {\e^h}{P[\e^h]}$. 
Then $\Lpnorm{1}{P}{g} = 1$ and $\Lpnorm{\infty}{P}{g} \le \e^{ \Lpnorm{\infty}{P}{h} }$.
Let $f(v) = \norm{v-w}_{\SigmaInv}$. 
Then 
$\EEE{v \sim Q} {\norm{v-w}_{\SigmaInv} }
= \Lpnorm{1}{Q}{f} = \Lpnorm{1}{P}{fg}$.
Finally,
let $\chi$ be the indicator function for the ellipsoid $\set{ v : \norm{v-w}_{\SigmaInv} \le R }$, 
and let $\bar\chi = 1 - \chi$.
Then $\Lpnorm{\infty}{P}{ f \chi} \le R$ and
\begin{align}
\Lpnorm{1}{P}{fg}   
&= \Lpnorm{1}{P}{ f g \chi }
      + \Lpnorm{1}{P}{ f g \bar \chi } \\
& \le  \Lpnorm{\infty}{P}{ f \chi} \, \Lpnorm{1}{P}{g} 
        + \Lpnorm{1}{P}{f \bar \chi} \, \Lpnorm{\infty}{P}{g}  
= R + 
         {\textstyle\sqrt{\frac {2}{\pi}}} \, \e^{-\frac{R^2}{2}} \e^{ \Lpnorm{\infty}{P}{h} },
\end{align}
where the inequalities follow from two applications of H\"older's inequality.
Choosing $R = \sqrt{2 \Lpnorm{\infty}{P}{h} }$ gives
$
\Lpnorm{1}{Q}{f} \le \sqrt{2 \Lpnorm{\infty}{P}{h} } + \sqrt{2/\pi}.
$
\end{proof}

\begin{proof}[Proof of \cref{twocontrol}] 
Let $P^S = N(w(S))$ and $P^*(S) = N(w^*(S))$.
By the closure of $\epsilon$-differential privacy under composition,
$P^*(S)$ is $\epsilon$-differentially private and is absolutely continuous with respect to $N(w)$ for all $w$, 
and so satisfies the conditions of \cref{nonprivate}.
In particular, with probability $1-\delta$,
\cref{DPpacbound} holds with 
$\KL{Q}{P^*(S)}$
replaced by $\KL{Q}{P^S} + Q \brackets{\ln \rnd{P^S}{P^*(S)} }$.

By hypothesis, with probability at least $1-\delta- \delta'$, 
it also holds that
$\norm{w(S)-w^*(S)}^2_2 \le C$.
Then, by \cref{normboundRN},
\begin{align}
Q \left[ \ln \rnd{P^S}{P^*(S)}\right]  
&\le \frac 1 2 \norm{w(S)-w^*(S)}^2_2/\mineig+  \norm{w(S)-w^*(S)}_2/\sqrt{\mineig}\ \EEE{v \sim Q} \norm{v-w(S)}_{\SigmaInv}  \\
&\le \frac 1 2 C/\mineig+ \sqrt{C/\mineig} \, \EEE{v \sim Q} \norm{v-w(S)}_{\SigmaInv}.
\end{align}
By \cref{GibbsboundRN}
$\EEE{v \sim Q} \norm{v-w(S)}_{\SigmaInv}$ is bounded for Gibbs measures based on a surrogate risk taking values in a length-$\Delta$ interval by $\sqrt{2 \tau \Delta} +\sqrt{2/\pi}$.
\end{proof}

\section{Bounded cross entropy}
\label{binobjective}

In order to achieve differential privacy, 
we work with a bounded version of the cross entropy loss.
The problem is associated with extreme probabilities near zero and one.
Our solution is to remap the probabilities $p \mapsto \psi(p)$, where
\begin{equation}\label{affine}
\psi(p) = \e^{-\Lmax} + ( 1- 2 \e^{-\Lmax}) p
\end{equation}
is an affine transformation that maps $[0,1]$ to $[\e^{-\Lmax},1-\e^{-\Lmax}]$, removing extreme probability values.
Cross entropy loss is then replaced by $g((p_1,\dots,p_K),y) = -\ln \psi(p_{y})$.
As a result, cross entropy loss is contained in the interval $[0,\Lmax]$.
We take $\Lmax = 4$ in our experiments.

\section{Computing PAC-Bayes bounds for Gibbs posteriors}
\label{KLeval}

For a given PAC-Bayes prior $P$ and dataset $S$, 
it is natural to ask which posterior $Q=Q(S)$
minimizes the PAC-Bayes bounds. In general,
some Gibbs posterior (with respect to $P$) is the minimizer. 
We now introduce the Gibbs posterior and discuss how we can
compute the term $\KL{Q}{P}$ in the case of Gibbs posteriors.

For a $\sigma$-finite measure $P$ over $\HH$
and function $g : \HH \to \Reals$, 
let $P[g]$ denote the expectation $\int g(h) P(\dee h)$ and, 
provided $P[g] < \infty$, let $P_{g}$ denote the probability measure on $\HH$, absolutely continuous with respect to $P$, with Radon--Nikodym derivative 
$
\rnd{P_{g}}{P}(h) = 
\frac{g(h)}{P[g]}.
$
A distribution of the form $P_{\exp \parens { - \tau g }}$  is generally referred to as a Gibbs distribution. 
A Gibbs \emph{posterior} is 
a probability measure of the form $P_{\exp \parens{ -\tau \EmpRisk{S}{} }}$ for some constant $\tau >0$.

The challenge of evaluating PAC-Bayes bounds for Gibbs posteriors is computing the KL term.
We now describe a classical estimate and show that it is going to be an upper bound with high probability.
Fix a prior $P$ and $\tau \ge 0$,
let $\Qt{\tau} = P_{\exp \parens{ -\tau \EmpRisk{S}{} }}$,
and let $\Zt{\tau} = P \brackets{ \exp \parens{ -\tau \EmpRisk{S}{} }}$. 
Then
\[
\KL{\Qt{\tau}}{P}
&= \Qt{\tau} \brackets[\Big]{ \ln \rnd{\Qt{\tau}}{P} } \\
&= \Qt{\tau} \brackets[\Big]{ \ln \frac {\exp \parens{- \tau \EmpRisk{S}{} }}{\Zt{\tau}} }\\
&= -\tau \Qt{\tau}[\EmpRisk{S}{}] - \ln \Zt{\tau}.
\]
Letting $W_1,\dots,W_n \sim \Qt{\tau}$,
we have 
\[
\Qt{\tau}[\EmpRisk{S}{}] 
= \sum_{i=1}^{n} \Qt{\tau}\brackets{ \EmpRisk{S}{} }
= \EE \brackets[\Big]{ \frac 1 n \sum_{i=1}^n \EmpRisk{S}{W_i} }.
\]
(The quantity within the expectation on the r.h.s.\ thus defines an unbiased estimator of $\Qt{\tau}$.)
In the ideal case, the samples are independent, and then the variance decays at an $n^{-1}$ rate.
In practice, it is often difficult to even sample from $\Qt{\tau}$ for high values of $\tau$. 
Indeed, using this approach, we would generally overestimate the risk, which means that we do not obtain an upper bound on the KL term. So instead, we approximate $ -\tau \Qt{\tau}[\EmpRisk{S}{}] \approx 0$. Despite this, we obtain nonvacuous bounds.
(For an alternative approach to this problem, see \citep{thiemann2017strongly}.)  

The second term is challenging to estimate accurately, even assuming that $P$ and $\Qt{\tau}$ can be
efficiently simulated.
One tack is to consider i.i.d.\ samples $V_1,\dots,V_n \sim P$, and note that
\[\label{logzdefn}
- \ln \Zt{\tau} = -\ln P[\exp \parens{-\tau \EmpRisk{S}{} }] 
&= -\ln \EE \brackets[\Big]{ \frac 1 n \sum_{i=1}^n \exp \parens {- \tau \EmpRisk{S}{V_i}} } \\
&\le \EE \brackets[\Big]{ -\ln  \frac 1 n \sum_{i=1}^n \exp \parens {- \tau \EmpRisk{S}{V_i}} },
\]
where the inequality follows from an application of Jensen's inequality. 
The quantity within the expectation on the r.h.s.\
thus forms an upper bound, and indeed, it is possible to show that it does not fall below the l.h.s.\ by $\epsilon$ with probability exponentially small in $\epsilon$.  
Thus we have a high-probability (near) upper bound on the term in the KL. One might be inclined to compute a normalized importance sampler, but since $Q$ cannot be effectively sampled, one does not obtain an upper bound with high probability.

The term $\ln \Zt{\tau} $ is a generalized log marginal likelihood, which, in our experiments, we approximate by sampling from a Gaussian distribution $P$. 
Numerical integration techniques rapidly diminish in accuracy with increasing dimensionality of the parameter space. 

Note, that due to the convexity of the exponential, samples $W_i \sim P$, for which $ \EmpRisk{S}{W_i}$ is close to zero, will dominate $\Zt{\tau}$. 
 Due to high dimensionality of the neural network parameter space, 
 with high probability a random sample $W_i$ from $P$ will not be far from minima of the empirical loss surface and therefore $\EmpRisk{S}{W_i}$ will be high. 
As a results, in our experiments we obtain a very loose upper bound on the KL.

\section{Experimental setup}
\label{app:setup}

\paragraph{Bounded loss}

While it is typical to train neural networks by minimizing cross entropy, this loss is unbounded
and our theory is developed only for bounded loss. 
We therefore work with a bounded version of cross-entropy loss, which
we obtain by preventing the network from producing extreme probabilities near zero and one. 
We describe our modification of the cross entropy in \cref{binobjective}.

\paragraph{Datasets}

We use two datasets. The first is MNIST, which consists of handwritten digit images with labels in $\{0,...,9\}$. The dataset contains 50,000 training images and 10,000 validation images. 

We also use a small synthetically generated dataset, which we refer to as \syn. 
The \syn{} dataset consists of 50 training data and 100 heldout data. Each input is a 4-dimensional vector sampled independently from a zero-mean Gaussian distribution with an identity covariance matrix. The true classifier is linear.  The norm of the separating hyperplane is  sampled from a standard normal. 

The random label experiments are performed on a dataset where the labels are independently and uniformly generated and thus the risk is 0.5 under 0--1 loss.

\paragraph{Architectures}

We use SGLD without any standard modifications (such as momentum and batch norm) to ensure that the stationary distribution is that of SGLD. For MNIST, we use a fully connected neural network architecture. The network has 3 layers and 600 units in each hidden layer. The input is a 784 dimensional vector and the output layer has 10 units. For the \syn{} dataset, we use a fully connected neural network with 1 hidden layer consisting of 100 units. The input layer has 4 units, and the output layer is a single unit.

\paragraph{Learning rate}

At epoch $t$, 
the learning rate is  $a_t = a_0*t^{-b}$, where $a_0$ is the initial learning rate and $b$ is the decay rate.
We set $b=0.5$ and use $a_0=10^{-5}$ for MNIST experiments and $a_0=10^{-3}$ for \syn{} experiments.

\paragraph{Minibatches} An epoch refers to the full pass through the data in mini batches of size 128 for MNIST data, and 10 for SYNTH data.

\end{document}

%% file: headers.tex
\usepackage[utf8]{inputenc} 
\usepackage{amsmath, amssymb,bm, cases, mathtools, thmtools}
\usepackage{verbatim}
\usepackage{graphicx}\graphicspath{{figures/}}
\usepackage{multicol}
\usepackage{caption}
\usepackage{mathrsfs} 
\usepackage{algorithm}
\usepackage{algorithmicx}
\usepackage[noend]{algpseudocode}
\usepackage{xifthen}

\usepackage[%
    minnames=1,maxnames=99,maxcitenames=3,
    style=authoryear, %
    doi=true,url=false,
    ibidtracker=false,
    giveninits=true,
    hyperref,natbib,backend=bibtex]{biblatex}
\renewbibmacro{in:}{%
  }
\DeclareNameAlias{sortname}{given-family}
\AtEveryBibitem{\clearfield{issn}}
\AtEveryBibitem{\clearfield{isbn}}
\AtEveryCitekey{\clearfield{issn}}
\AtEveryCitekey{\clearfield{isbn}}
\AtEveryBibitem{\ifentrytype{book}{\clearfield{pages}}{}}
\bibliography{biblio}

\usepackage{hypernat}
\usepackage{datetime}

\makeatletter
\let\reftagform@=\tagform@
\def\tagform@#1{\maketag@@@{\ignorespaces\textcolor{gray}{(#1)}\unskip\@@italiccorr}}
\renewcommand{\eqref}[1]{\textup{\reftagform@{\ref{#1}}}}
\makeatother

\input{commenting.tex}

%% file: commenting.tex
\newcommand{\LATER}[1]{\error}
\newcommand{\fLATER}[1]{\error}
\newcommand{\TBD}[1]{\error}
\newcommand{\fTBD}[1]{}
\newcommand{\PROBLEM}[1]{\error}
\newcommand{\fPROBLEM}[1]{\error}

%% file: defs.tex
\def\[#1\]{\begin{align}#1\end{align}}
\def\*[#1\]{\begin{align*}#1\end{align*}}

\def\clap#1{\hbox to 0pt{\hss#1\hss}}

\newcommand{\defas}{\vcentcolon=}  %

\newcommand{\Lmax}{\ell_{\text{max}}}

\newcommand{\Reals}{\mathbb{R}}

\newcommand{\Nats}{\mathbb{N}}

\newcommand{\NNReals}{\Reals_+}

\newcommand{\dee}{\mathrm{d}}

\DeclareMathOperator*{\newlim}{\mathrm{lim}\vphantom{\mathrm{infsup}}}

\DeclareMathOperator*{\newinf}{\mathrm{inf}\vphantom{\mathrm{infsup}}}
\DeclareMathOperator*{\newsup}{\mathrm{sup}\vphantom{\mathrm{infsup}}}
\renewcommand{\lim}{\newlim}

\renewcommand{\inf}{\newinf}
\renewcommand{\sup}{\newsup}

\newcommand{\ProbMeasures}[1]{\mathcal{M}_1(#1)}

\renewcommand{\Pr}{\mathbb{P}}
\def\EE{\mathbb{E}}

\newcommand{\defn}[1]{\emph{#1}}

\newcommand{\KLname}{\mathrm{KL}}
\newcommand{\HH}{\mathcal H}
\newcommand{\KL}[3][]{\KLname#1(#2 #1|#1|#3 #1)}
\newcommand{\KLnamebin}{\mathrm{kl}}
\newcommand{\KLbin}[3][]{\KLnamebin#1(#2 #1|#1|#3 #1)}

\newcommand{\Bernoulli}[1]{\mathcal B_{#1}}

\newcommand{\loss}{\ell}

\newcommand{\e}{\mathrm{e}}

\newcommand{\eqdist}{\overset{{}_{d}}{=}}
\newcommand{\randto}{\rightsquigarrow}
\newcommand{\ww}{\mathbf{w}}

\renewcommand{\HH}{\Reals^p}
\newcommand{\RHH}{\ProbMeasures{\HH}}
\newcommand{\Dist}{\mathcal D}
\newcommand{\Alg}{\mathscr A}
\newcommand{\PAlg}{\mathscr P}

\newcommand{\EEE}[1]{\underset{#1}{\EE}}
\newcommand{\PPr}[1]{\underset{#1}{\Pr}}

\newcommand{\set}[2][]{#1\{#2 #1\}}
\newcommand{\brackets}[2][]{#1[#2 #1]}
\newcommand{\parens}[2][]{#1(#2 #1)}
\newcommand{\abs}[2][]{#1\lvert#2 #1\rvert}
\newcommand{\norm}[2][]{#1\lVert#2 #1\rVert}
\newcommand{\event}[2][]{#1\lbrace#2 #1\rbrace}

\newcommand{\tuple}[2][]{#1 \langle #2 #1 \rangle}
\newcommand{\ip}[3][]{\tuple[#1]{#2,#3}}

\renewcommand{\defas}{\overset{\text{\smash{\tiny{def}}}}{=}}

\newcommand\optparen[1]{\ifthenelse{\equal{#1}{}}{}{(#1)}}
\newcommand{\RiskChar}{L}
\newcommand{\Risk}[2]{\RiskChar_{#1}\optparen{#2}}
\newcommand{\EmpRisk}[2]{\hat \RiskChar_{#1}\optparen{#2}}
\newcommand{\SurRisk}[2]{\tilde \RiskChar_{#1}\optparen{#2}}

\newcommand{\EmpDist}{\hat\Dist}

\newcommand{\Gibbs}[2]{#1_{\exp\parens{#2}}}

\newcommand{\rnd}[2]{\frac{\dee #1}{\dee #2}}

\newcommand{\amaxinf}[3]{I_{\infty}^{#1}(#2;#3)}
\newcommand{\maxinf}[2]{\amaxinf{}{#1}{#2}}
\newcommand{\amaxinfalg}[3]{I_{\infty}^{#1}(#2,#3)}
\newcommand{\maxinfalg}[2]{\amaxinfalg{}{#1}{#2}}

\newcommand{\ppWD}{\mathcal W}
\newcommand{\pWD}[4][]{\ppWD_{#2} #1(#3,#4 #1)}

\newcommand{\syn}{SYNTH}